\documentclass{article}
\pdfoutput=1

\usepackage[english]{babel}

\usepackage[letterpaper,top=2cm,bottom=2cm,left=3cm,right=3cm,marginparwidth=1.75cm]{geometry}

\usepackage[colorlinks=true, allcolors=blue]{hyperref}
\usepackage{natbib}

\usepackage{xcolor}
\usepackage{parskip}
\usepackage{graphicx}
\usepackage{algpseudocode}
\usepackage{amsmath,amsthm,amssymb,amsfonts,dsfont,bm}
\usepackage{subcaption}
\usepackage{float}
\usepackage{enumitem}
\usepackage[ruled,vlined]{algorithm2e}

\usepackage[capitalize,noabbrev]{cleveref}
\usepackage{booktabs}
\usepackage{tabularx}
\usepackage{makecell}

\allowdisplaybreaks

\theoremstyle{plain}
\newtheorem{theorem}{Theorem}[section]

\newtheorem{lemma}[theorem]{Lemma}

\theoremstyle{definition}

\newtheorem{assumption}[theorem]{Assumption}
\theoremstyle{remark}

\title{\textbf{Beyond Precision: Training-Inference Mismatch is an Optimization Problem and Simple LR Scheduling Fixes It}}
\author{
Yaxiang Zhang$^{1}$\thanks{Equal contribution. The first two authors are listed in alphabetical order.},~
Yingru Li\footnotemark[1]~\thanks{Corresponding author: Yingru Li, \texttt{szrlee@gmail.com}.},~
Jiacai Liu$^{2}$,~
Jiawei Xu$^{3}$,~
Ziniu Li$^{3}$,~
Qian Liu,~
Haoyuan Li$^{4}$ \vspace{0.1in}\\
 $^{1}$ National University of Singapore, ~~$^{2}$ Fudan University,  ~~$^{3}$ CHUK-Shenzhen
 \\ $^{4}$ University of Science and Technology of China
\date{}
}
\begin{document}
\maketitle

\begin{abstract}
Reinforcement Learning (RL) for training Large Language Models is notoriously unstable. While recent studies attribute this to ``training-inference mismatch'' stemming from inconsistent hybrid engines, standard remedies—such as Importance Sampling—might fail during extended training runs. In this work, we analyze this instability through the lens of optimization, demonstrating that gradient noise and training-inference mismatch escalate in tandem as training progresses. Meanwhile, we find that the mismatch can be effectively suppressed by shrinking the update size. Taken together, we deduce that the mismatch is not merely a static numerical discrepancy, but a dynamic failure coupled with the model's optimization. Based on this insight, we propose a simple yet effective solution: a specialized Learning Rate (LR) scheduler. Instead of pre-defined decay schedule in traditional LR scheduler, our method dynamically triggers LR decay based on response length, which we identify as a reliable early-warning signal for impending instability. Empirical evidence suggests that by reducing the learning rate as gradient noise rises, we can consistently stabilize RL training and keep the training-inference mismatch at a safe level.
\end{abstract}

\section{Introduction} \label{Sec:Introduction}
Reinforcement Learning (RL) has emerged as a pivotal paradigm for fine-tuning Large Language Models (LLMs), demonstrating remarkable success in increasing complex reasoning capabilities \citep{guo2025deepseek,rafailov2023direct,xie2025logic}. Despite these empirical gains, RL training is notoriously fragile. The optimization process often encounters abrupt ``collapses''—sharp declines in training rewards and validation accuracy before reaching potential convergence \citep{yao2025your,liuspeed,zheng2025group,cui2025entropy}. Such instability not only hinders the full exploitation of high-quality training data but also severely compromises sample efficiency and computational economy.

A growing body of literature identifies \textbf{the training-inference mismatch} as a primary catalyst for this instability \citep{yao2025your,liuspeed,team2025every}. This discrepancy is an inherent byproduct of the heterogeneous engine utilized in large-scale RL: while specialized inference kernels (e.g., vLLM~\citep{kwon2023efficient}/SGLang~\citep{zheng2024sglang}) are deployed to accelerate massive rollout generation, the subsequent gradient computation necessitates distributed training frameworks (e.g., FSDP~\citep{zhao2023pytorch}/Megatron~\citep{shoeybi2019megatron}). Due to the \textit{non-associativity} of floating-point arithmetic in finite precision, the distinct operation sequences across these engines manifest as subtle numerical variances. While seemingly negligible, these discrepancies are amplified by low-precision formats (e.g., BF16/FP16), hardware-specific kernel optimizations, and the inherent non-linearity of deep architectures, which could finally lead to significant noise in the optimization process. 

To mitigate this, previous studies have predominantly relied on Importance Sampling (IS) variants. For instance, \citep{yao2025your} introduces token-level IS to correct the likelihood ratio between training and inference engines, while \citep{liuspeed} proposes sequence-level corrections, arguing that token-level adjustment is insufficient to fully stabilize training due to the wrong correction level. However, these methods typically necessitate heuristic ``masking'' or ``clipping'' to suppress the influence of dangerous extreme samples. As we'll discuss later, this leaves fundamental flaw in their design and can even precipitate catastrophic collapse in extended training runs.

In this paper, we challenge the prevailing characterization of training-inference mismatch as a \textit{static} numerical artifact. While such discrepancies might initially appear as state-independent random noise stemming from precision limits, our analysis reveals that the magnitude of the mismatch is deeply coupled with the model's location in the weight space. Crucially, we observe that the mismatch can be effectively suppressed by modulating the update magnitude, suggesting a dynamic interplay between optimization dynamics and numerical stability. 

Our main contributions are summarized as follows:

\begin{itemize}
\item \textbf{Dynamic Instability Analysis}: We empirically demonstrate that training-inference mismatch is not a stationary background noise. Instead, it is a dynamic failure mode that escalates in tandem with gradient noise and the evolution of the optimization of the model.
\item \textbf{Critique of Epoch-based Scheduling}: We show that traditional Learning Rate (LR) schedulers, governed by predefined epoch counts, are ill-suited for RL training. We provide evidence that the epoch count is a poor proxy for the Signal-to-Noise Ratio (SNR) of gradients and the optimization dynamics, leading to suboptimal or premature LR decay.
\item \textbf{Adaptive length-decay Scheduler}: We propose a novel, reactive LR scheduler that dynamically triggers decay based on a newly identified ``early-warning signal'': the surge in average response length. It halves the learning rate every \texttt{decay\_period} step, until it reaches the preset minimal learning rate. Our results suggest that a \texttt{decay\_period} proportional to this surge ($\approx 1.8\times$) consistently stabilizes training and maintains the mismatch at a safe level.
\end{itemize}

This paper is organized as follows: ~\cref{Sec:Background} introduces the background about training-inference mismatch and other related works; ~\cref{Sec:Method} details our motivation and the proposed methodology; ~\cref{Sec:experiment} presents comprehensive experimental evaluations; ~\cref{Sec:conclusion} summarizes our findings and proposes some possible directions for future work.

\section{Background and Related Work} \label{Sec:Background}
\subsection{Training-Inference Mismatch and Importance Sampling Correction}
The goal of RL is to optimize the following objective over training policy \( \textcolor{blue}{\pi}_{\theta}(y|x)\):
\begin{equation}
    \mathcal{J}(\theta) = \mathbb{E}_{x \sim p_x} \left[ \mathbb{E}_{y \sim \textcolor{blue}{\pi}_{\theta}}[R(x,y)] \right],
\end{equation}
where \(x\) is the prompt sampled from a distribution \(p_x\) and $y=(y_1, y_2, \cdots, y_T)$ is the generated response with horizon $T$. The reward function, $R(x, y)$, quantifies the quality or correctness of the response $y$ given the prompt $x$. 
To optimize this objective, the policy gradient is typically estimated using the REINFORCE estimator \citep{williams1992simple}:
\begin{equation} \label{eq:ideal graident}
    \nabla_{\theta} \mathcal{J}(\theta) = \mathbb{E}_{x \sim p_x}  \mathbb{E}_{y \sim \textcolor{blue}{\pi}_{\theta}} \left[ \nabla_{\theta} \log \textcolor{blue}{\pi}_{\theta}(y|x) R(x,y)\right].
\end{equation}
However, in modern RL frameworks for LLM fine-tuning, different engines are used for rollout and update to maximize system efficiency. As mentioned above, finite-precision floating-point arithmetic does not satisfy the associative law, hence different calculation orders in different engines would lead to discrepancy between training policy \( \textcolor{blue}{\pi}_{\theta}\) (used for update) and rollout policy \( \textcolor{red}{\mu}_{\theta}\). Hence, the real gradient is:
\begin{equation} \label{eq:biased gradient}
    \nabla_{\theta} \mathcal{J}_{actual}(\theta) = \mathbb{E}_{x \sim p_x}  \mathbb{E}_{y \sim \textcolor{red}{\mu}_{\theta}} \left[ \nabla_{\theta} \log \textcolor{blue}{\pi}_{\theta}(y|x) R(x,y)\right].
\end{equation}
This discrepancy introduces an off-policy bias, as the gradient is evaluated on trajectories sampled from $\textcolor{red}{\mu}_{\theta}$ rather than the target policy $\textcolor{blue}{\pi}_{\theta}$.  A natural and principal solution for this mismatch is to use Importance Sampling(IS). IS is a classical technique for obtaining an unbiased estimate of an expectation when the sampling distribution differs from the target distribution. The corrected estimation of policy gradient is:
\begin{equation}
    \nabla_{\theta} \mathcal{J}_{is}(\theta) = \mathbb{E}_{x \sim p_x, y \sim \textcolor{red}{\mu}_{\theta}}   \left[ \rho \cdot \nabla_{\theta} \log \textcolor{blue}{\pi}_{\theta}(y|x) R(x,y)\right],
\end{equation}
where \( \rho = \frac{\textcolor{blue}{\pi}_{\theta}(y|x)}{\textcolor{red}{\mu}_{\theta}(y|x)}\) denotes the probability ratio.

There are still practical questions left for this correction. The first one is -- we should apply IS on what level? i.e, what does response \(y\) exactly mean? \citep{yao2025your} argues that we can take \(y\) as a single token, calculate probability ratio for each token and apply IS in a response sequence auto-regressively. We refer to this method as token-level IS. However, \citep{liuspeed} points out that token-level IS is still biased under a sequence perspective. To fix this, they propose sequence-level(seq-level) IS, i.e treat the whole response sequence as \(y\). \\
Whether token-level or seq-level, vanilla IS is still infeasible in practice due to high variance, especially in the LLM-RL training where response sequences are long, leading to extreme probability ratios. To curb this, some extra techniques to trade variance with bias are usually necessary, such as Truncated Importance Sampling (TIS) \citep{espeholt2018impala,yao2025your} or Masked Importance Sampling (MIS) \citep{liuspeed,zheng2025group}:
\begin{equation}
    \nabla_{\theta} \mathcal{J}_{tis}(\theta) = \mathbb{E}  \left[ \min \left( \rho , C\right) \nabla_{\theta} \log \textcolor{blue}{\pi}_{\theta}(y|x) R(x,y) \right],
\end{equation}
\begin{equation}
    \nabla_{\theta} \mathcal{J}_{mis}(\theta) = \mathbb{E}  \left[ \rho \cdot \mathbb{I}{(\rho \le C)} \nabla_{\theta} \log \textcolor{blue}{\pi}_{\theta}(y|x) R(x,y) \right],
\end{equation}
where \(C\) is a clipping hyperparameter, \( \mathbb{I} \{ \cdot \}\) is the indicator function, \(x,y\) are drawn from \(x \sim p_x, y \sim \textcolor{red}{\mu}_{\theta}(y|x)\) to compute expectation. This remedy introduces a new hyperparameter \(C\) which requires extra tuning. Moreover, the truncation or mask in fact reintroduces bias in gradient estimation and breaks the theoretical soundness. Our experiments show that these corrections sometimes are only able to extend stable training window and still suffer from possible training collapse.

\subsection{Related Work}

\textbf{Other stability techniques} Beyond importance sampling, other research attempts to mitigate training-inference mismatch through engineering workarounds. These include adjusting numerical precision——such as employing FP32 heads \citep{chen2025minimax} or switching between BF16 to FP16 formats \citep{qi2025defeating}——though these have proven insufficient to prevent training collapse in prior work\citep{yao2025your,liuspeed} and our ablation experiment. While manual alignment of training and inference implementations has achieved progress \citep{team2025every}, such "bespoke" fixes demand substantial engineering efforts, and generalization across evolving frameworks or architectures remains unclear.

The persistence of this mismatch is rooted in fundamental computational differences. For instance, the transition from auto-regressive generation during inference to parallel processing during training alters the sequence of floating-point operations. The inherent discrepancies suggest that engineering workarounds are insufficient and inefficient, highlighting the need for a more robust solution grounded in optimization dynamics.

\textbf{Learning Rate Scheduler} Learning Rate (LR) scheduling is a cornerstone of deep learning optimization, which critically influences both convergence speed and generalization. The conventional rationale for LR decay is rooted in the dynamics of stochastic optimizers (e.g., SGD~\citep{robbins1951stochastic}, AdamW~\citep{loshchilov2017decoupled}): an initially high learning rate facilitates rapid exploration and helps the model escape sharp, spurious local minima, while subsequent decay ensures stable convergence and suppresses gradient-induced oscillations \citep{you2019does}. To this end, various predefined schedules have been developed, such as Cosine Annealing \citep{loshchilov2016sgdr}, Linear Decay \citep{devlin2019bert}, and Piecewise Constant Decay \citep{he2016deep}.

However, Reinforcement Learning for LLMs introduces unique challenges that traditional schedulers fail to address. First, as we'll show in Section \ref{Sec:Method}, the ``optimal'' training duration in RL is often unpredictable. Second, conventional schedulers are ``blind'' to the internal states of the model and system-level discrepancies. Our work fills this gap by proposing a scheduler that is explicitly reactive to early-warning signals of numerical instability, moving beyond static, time-based decay patterns.

\section{Methodology} \label{Sec:Method}
\subsection{Why LR Scheduler}
We encounter training collapse when conducting math RL experiments on Qwen3 Base models \citep{yang2025qwen3}. To quantitatively diagnose the degree of training-inference mismatch and its causal link to optimization failure, we firstly introduce the Log Perplexity (\(\log \text{ppl}\)) of a trajectory \( \tau=(x,y)\) given the model weight \( \theta\)
\begin{equation}
    \log \text{ppl}(\tau, \pi_{\theta})= -\sum_{t=1}^{T}  \log \pi_{\theta}(y_t|x,y_{<t}).
\end{equation}
Then we define the mismatch indicator \(\log\text{ppl\_abs\_diff}\) :
\begin{equation}
    \log\text{ppl\_abs\_diff} = \frac{1}{N} \sum_{i=1}^{N}|\log\text{ppl}(\tau_i,\textcolor{blue}{\pi}_\theta)-\log\text{ppl}(\tau_i, \textcolor{red}{\mu}_\theta)|,
\end{equation}

where \(N\) is the batch size in each update step.

\begin{figure}[H]
    \centering
    \includegraphics[width=1\linewidth]{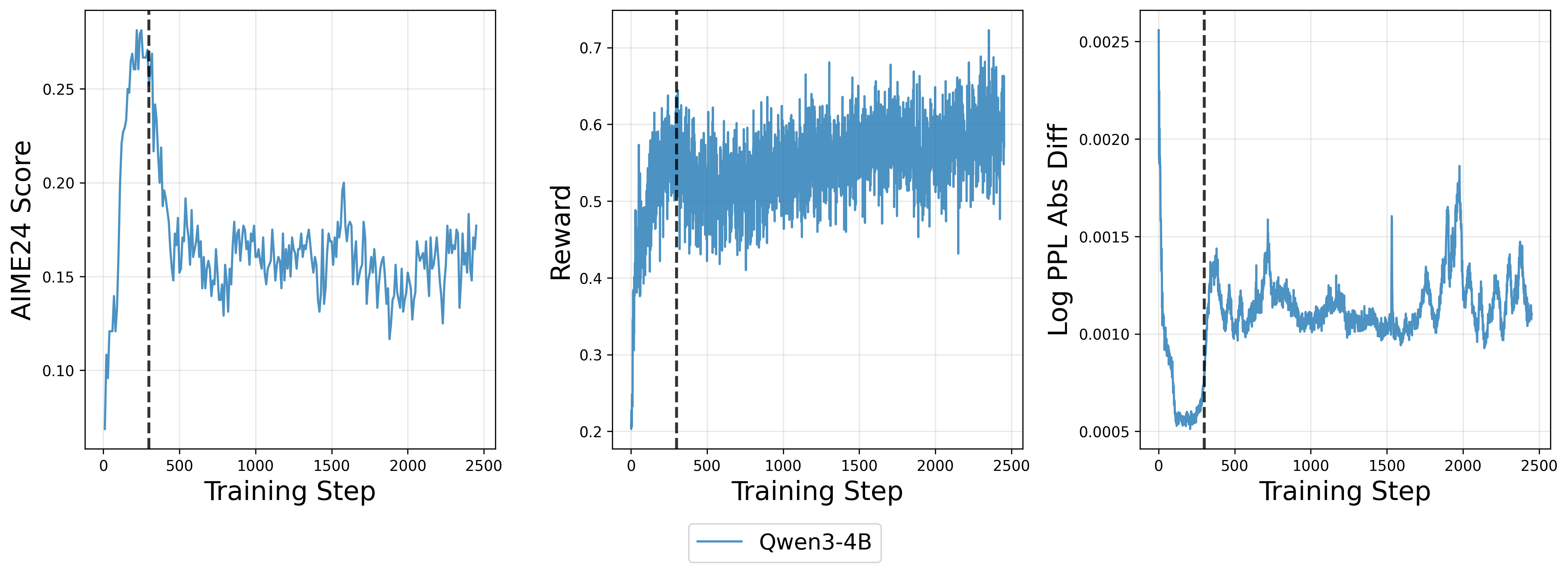}
    \caption{Results of Qwen3-4B-Base on filtered DAPO training dataset. The vertical dotted line indicates the point (300 step) at which the methods collapse.}
    \label{fig:baseline}
\end{figure}

\noindent We monitor $\log \text{ppl\_abs\_diff} $ during the training process and as Figure \ref{fig:baseline} shows, the onset of the distribution mismatch coincides with the collapse of validation accuracy(both around 300 step). This observation corroborates the patterns reported in previous studies \citep{yao2025your,liuspeed}. Following suggested remedies, we applied various IS patches. However, as Figure \ref{fig:importance sampling patches} shows, while these techniques successfully prolonged the training process, only token-level TIS prevents collapse; most of them eventually failed in this experiment with a more thorough collapse compared to baseline.

\begin{figure}
    \centering
    \includegraphics[width=1\linewidth]{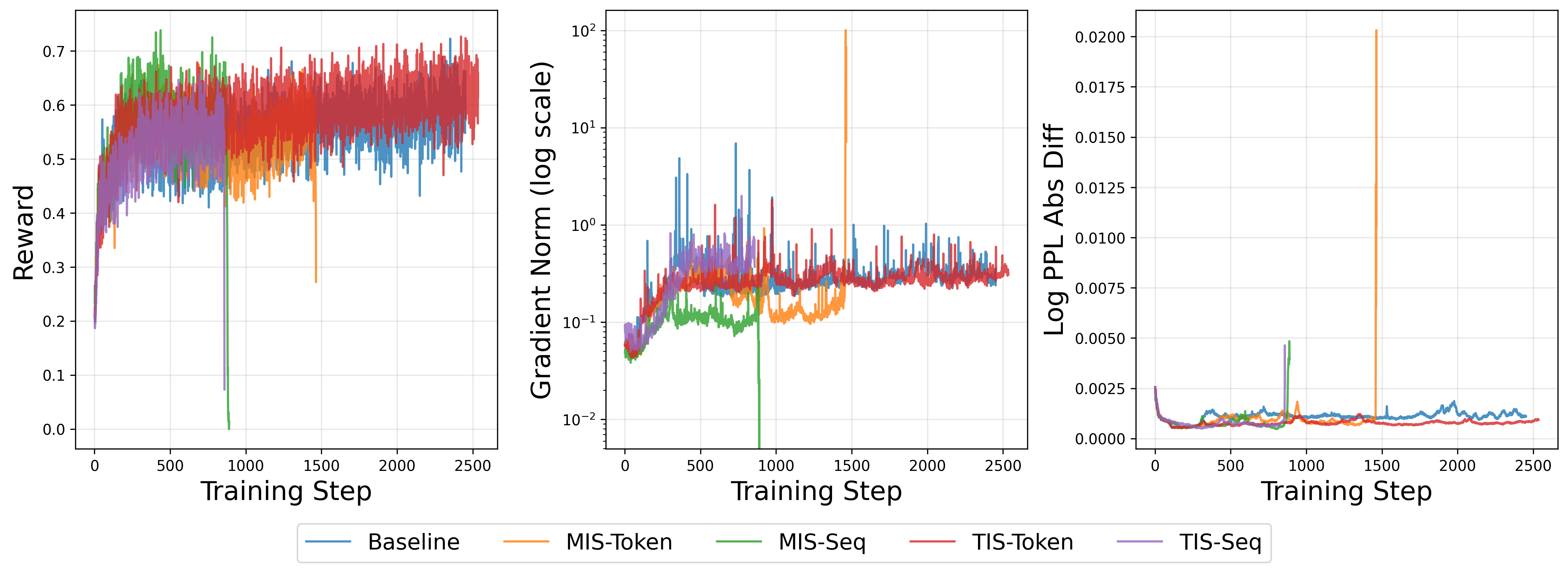}
    \caption{Training stability comparison with/without IS patches.}
    \label{fig:importance sampling patches}
\end{figure}

\noindent This suggests that conventional IS techniques still have significant limitations in this context.
To investigate whether this emerging challenge—training-inference mismatch in RL can be addressed by ``old wisdom'' from optimization theory, we analyze the smoothed gradient norm.

\begin{figure}[H]
    \centering
    \includegraphics[width=0.95\linewidth]{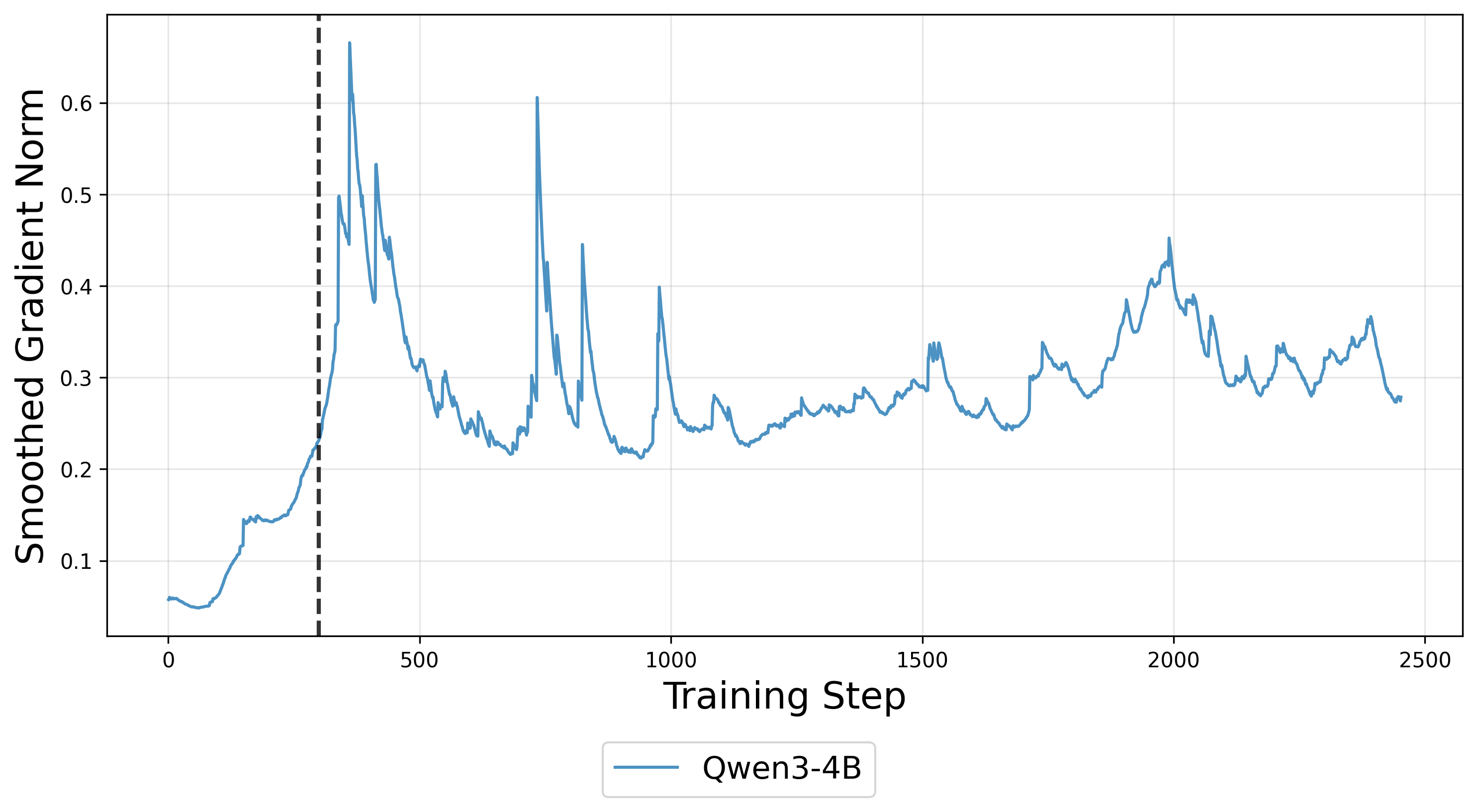}
    \caption{Smoothed gradient norm.  The vertical dotted line is at 300 step.}
    \label{fig:smooth-grad-norm}
\end{figure}

\noindent The gradient norm reflects a composite of signal and noise. As the model fits the training data, the true gradient signal is expected to diminish; but the gradient norm escalates in Figure \ref{fig:smooth-grad-norm}. This suggests that the gradient estimation becomes increasingly dominated by noise rather than signal. Simple theoretical analysis reveals that reducing learning rate $\eta$ can drastically shrinks the effect brought by gradient noise (including bias and variance, detailed in Appendix \ref{Appendix:Discussion}). Inspired by this, we reduced our constant learning rate from the default \texttt{1e-6} to \texttt{1e-7}, and the experiment result is presented in Figure \ref{fig:lr1e-7}.

\begin{figure}[H]
    \centering
    \includegraphics[width=1\linewidth]{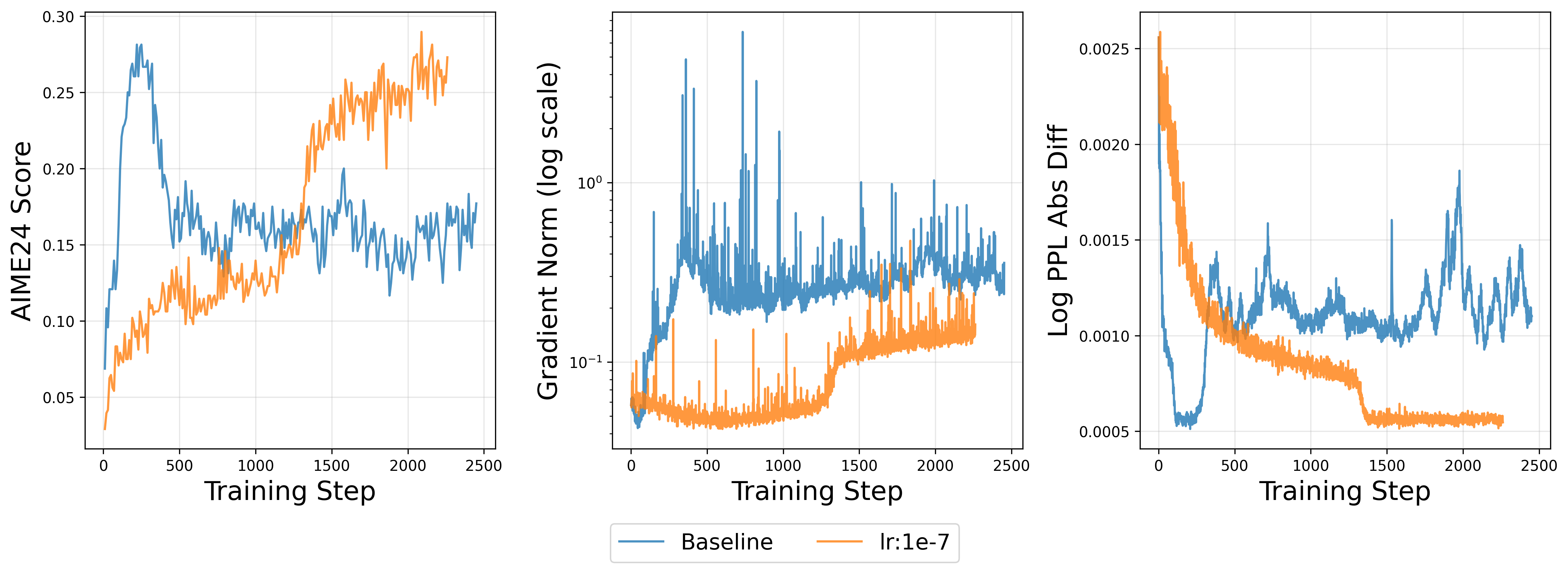}
    \caption{LR:1e-6 vs 1e-7. It’s worth noting that the blue line isn’t simply 10x slower than the green line. The green line keeps rapid and homogeneous growing and then rapidly collapses; the blue line, in contrast, has a sudden surge around 1.5k steps, which exhibits heterogeneity in the growing process.}
    \label{fig:lr1e-7}
\end{figure}

\noindent This phenomenon reveals some deeper facts. While one might expect such a discrepancy to stem solely from numerical precision limits (and thus behave as state-independent random noise), our experiments suggest that this mismatch can be effectively suppressed by shrinking the update size, which demonstrates its connection to model weight location and optimization dynamics. We hypothesize that in the later stages of training, malignant dynamics leads the model weights to regions with specific geometric characteristics(e.g, sharp region, higher curvature), which in turn amplifies numerical discrepancies. This interplay between numerical stability and loss landscape geometry remains an intriguing open question. 

Nonetheless, as globally lowering the learning rate inevitably compromises efficiency during the benign early stages, these observations provide a strong motivation for a dynamic LR scheduler.

\subsection{Traditional Scheduler is Not Feasible}

The application of traditional learning rate (LR) schedulers in RL raises a fundamental challenge: the absence of an informed ``stopping point'' or decay horizon. Predefined schedules lack the adaptability required to respond to real-time training dynamics. Specifically, a fixed schedule may decay the LR prematurely during the benign early stages, hindering rapid progress, or intervene too late—failing to arrest the optimization collapse once numerical instability takes hold. 

\begin{figure}[H]
    \centering
    \includegraphics[width=1\linewidth]{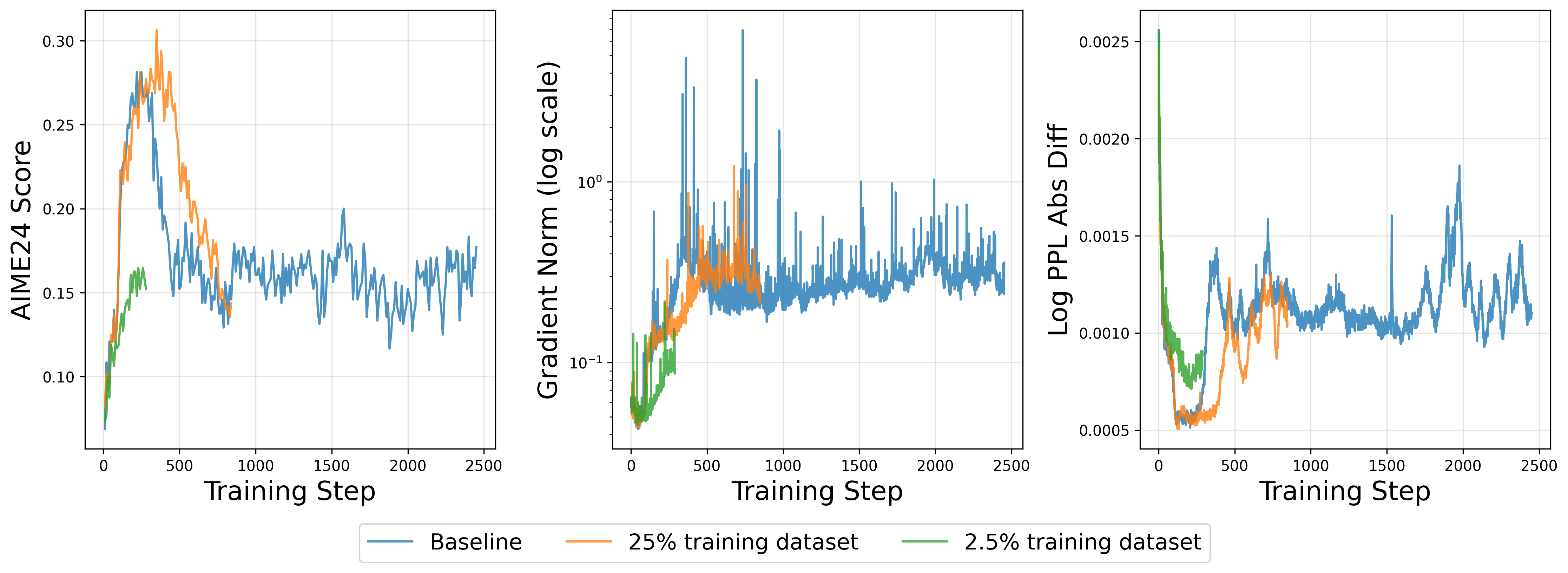}
    \caption{Experiment with partial training dataset. The collapse of orange line with 25\% dataset is even slower than using full dataset. The green line with 2.5\% dataset still grows for 300 steps.} 
    \label{fig:partial dataset}
\end{figure}

\noindent One might argue that the training duration is typically governed by the number of epochs, which could serve as a proxy for the Signal-to-Noise Ratio (SNR) of the gradients. From this perspective, identifying the appropriate epoch count for a specific class of task might suffice to justify traditional schedulers. However, in Figure \ref{fig:partial dataset}, our experiments using varying fractions of the complete dataset demonstrate that the dynamics of training and the onset of collapse are not linearly proportional to the size of the dataset. Even with a significant reduction in dataset size (orange line), the training dynamics and the eventual point of failure remain remarkably consistent. This suggests that RL training stability is roughly independent of data volume after reaching a ``saturation threshold''. Consequently, epoch count is an unreliable metric for scheduling, highlighting the need for a reactive mechanism triggered by internal signals rather than external time-based schedulers.

\subsection{Our Method: Decay by Response Length}

Our proposed LR scheduler adopts a reactive, heuristic-driven approach: it maintains the initial learning rate to maximize exploration until a signal of potential instability is detected. Upon triggering, the learning rate undergoes a length-decay, halving every \texttt{decay\_period} until it reaches a predefined floor (defaulting to 10\% of $\eta_0$). The operational logic is formalized in Algorithm \ref{alg:lrdecay}:

\begin{algorithm}[H] \label{alg:lrdecay}
\SetKwInOut{Input}{Input}
\SetKwInOut{Output}{Output}

\Input{Initial learning rate $\eta_0$, Final learning rate $\eta_{\infty}$, Decay period $T_{decay}$, Total training steps $T$}
\Output{Learning rate $\eta_t$ for each step $t$}

Initialize $\eta_t \leftarrow \eta_0$\;

\For{$t = 1$ \KwTo $T$}{   
    \If{$t \pmod{T_{decay}} == 0$}{
        $\eta_t \leftarrow \max(\frac{\eta_{t-1}}{2}, \eta_{\infty})$\;
    }
    
    Update model parameters using $\eta_t$\;
}
\caption{length-decay LR Scheduler}
\end{algorithm} 

\noindent A critical challenge lies in determining the optimal \texttt{decay\_period} pro-actively. To predict instability before its onset, we investigated various RL training metrics and identified the average response length as a reliable early-warning signal. Response length is an important metric and has been widely observed to correlate with training stability in RL \citep{yuan2025efficient,dai2025stable,singhal2023long,chai2024ma}. In particular, we observed a distinct pattern termed "response length surge," in which the average length can triple in just a few steps.

\begin{figure}[H]
    \centering
    \includegraphics[width=0.95\linewidth]{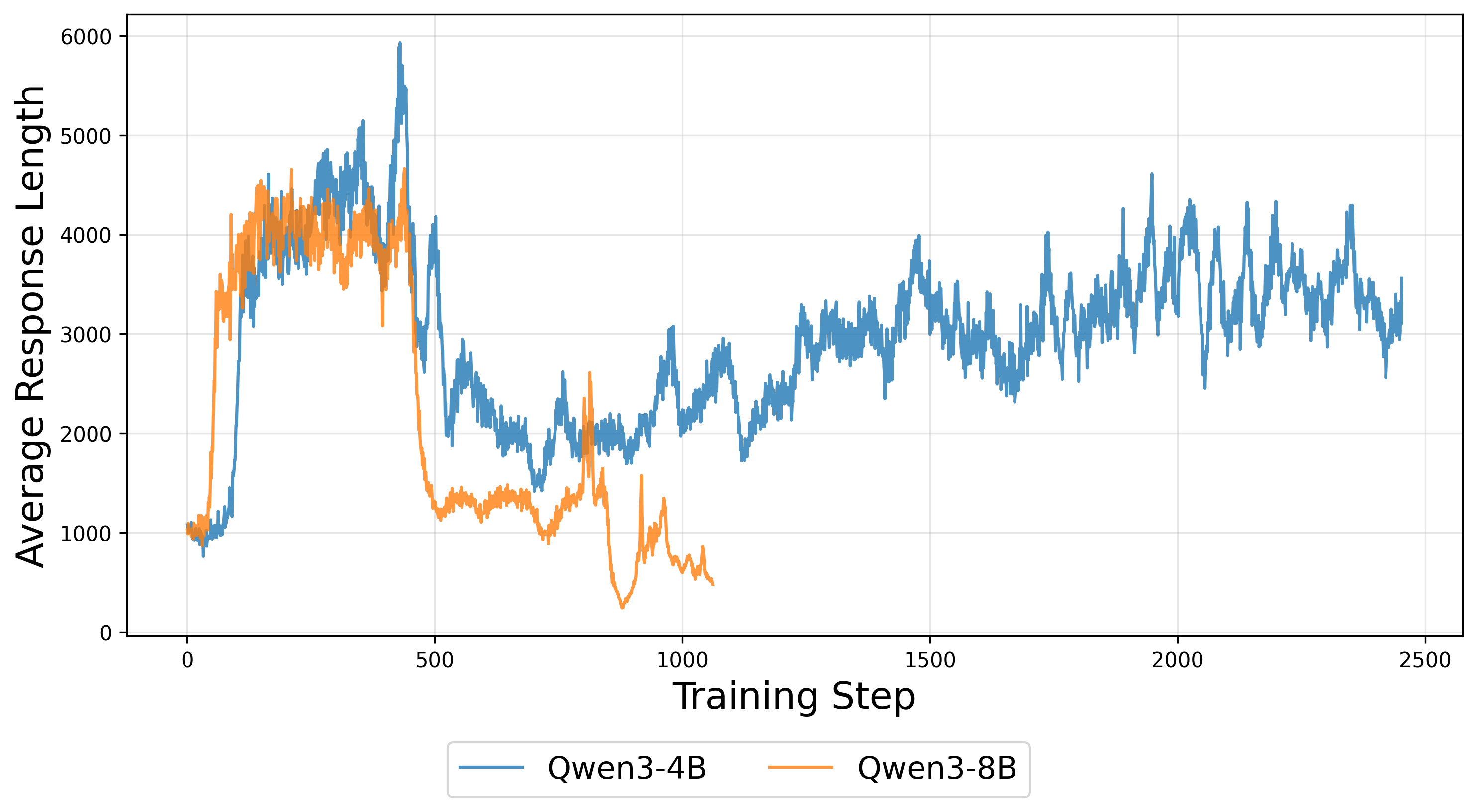}
    \caption{The average response length in Qwen3-4B-Base and Qwen3-8B-Base training. We can observe a clear surge around 100 step for both, from 1000 to 3000-4000.}
    \label{fig:Response length surge}
\end{figure}

\noindent This surge serves as an intuitive precursor to instability. In auto-regressive generation, while increased length may improve performance, it also inherently amplifies numerical discrepancies. To formalize why an expanding horizon $T$ could influence training dynamics, we derive an error bound on the gradient shift under mild regularity assumptions:

\begin{theorem}[informal] \label{thm:horizon_informal}
    Let \(\nabla_{\theta} \mathcal{J}(\theta)\) and \(\nabla_{\theta} \mathcal{J}_{actual}(\theta)\) be the ideal gradient for backpropagation and the real calculated gradient defined in \ref{eq:ideal graident} and \ref{eq:biased gradient}, respectively. Then, under mild assumptions, the following inequality holds:
\begin{equation*}
    \left \| \nabla_{\theta} \mathcal{J}_{actual}(\theta)-\nabla_{\theta} \mathcal{J}(\theta) \right \|_2 \leq C \cdot T^2.
\end{equation*}
    We provide detailed proof in Appendix \ref{Appendix:proof}.
\end{theorem}

The constant \(C\) represents the degree of numerical mismatch of a single-token and the amplitude of per-token score function, which are independent of sequence length. This quadratic scaling with horizon \(T\) comes from two parts: one is that the state visitation distribution at time \(t\) divergence accumulates linearly with time; one is the summation of per-step distribution divergences. This bound matches the classical bound on the relationship between effective horizon and penalty brought by using off-policy surrogate in TRPO algorithm \citep{schulman2015trust}. It reveals that the surge in response length alone may significantly amplify discrepancies brought by off-policy(including distribution mismatch).

However, this stability is not an instant effect: the numerical mismatch itself is small in magnitude, and it takes time for model weights to come into the "bad region". Based on our empirical observations, we propose to set the \texttt{decay\_period} as $1.8 \times$ \textbf{time of the response length surge} to balance the learning speed and training stability. Detailed experimental evidence supporting these heuristics will be presented in the subsequent section.

\section{Experiments} \label{Sec:experiment}

We evaluate our proposed method on Qwen3 Base models \citep{yang2025qwen3} with filtered DAPO dataset \citep{yu2025dapo} (around 13k samples) on VeRL framework \citep{sheng2025hybridflow} for RL training. Initial learning rate is \texttt{1e-6} and train batch size is \(64\). Comprehensive hyperparameter configurations are detailed in Appendix \ref{Appendix:hyperparameter}.

\subsection{Our Scheduler Stabilizes Training and Reduces Mismatch}

To assess the efficacy of the adaptive length-decay scheduler, we conduct a comparative analysis of training rewards and validation accuracy between the baseline and our method. Furthermore, we monitor the trajectory of the gradient norm and the $\log\text{ppl\_abs\_diff}$ indicator to evaluate training stability.

\begin{figure}[H]
    \centering
    \includegraphics[width=1\linewidth]{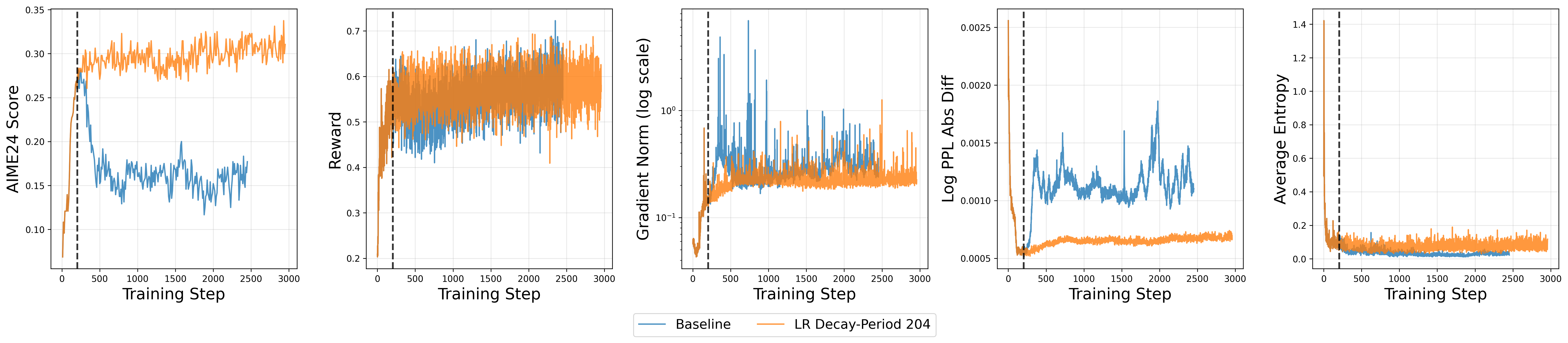}
    \caption{Comparison of applying length-decay LR Scheduler or not on Qwen3-4B-Base. The vertical dotted line is at 204 step, indicating the starting point of decaying LR.}
    \label{fig:4B-lrdecay-single}
\end{figure}
\begin{figure}[H]
    \centering
    \includegraphics[width=1\linewidth]{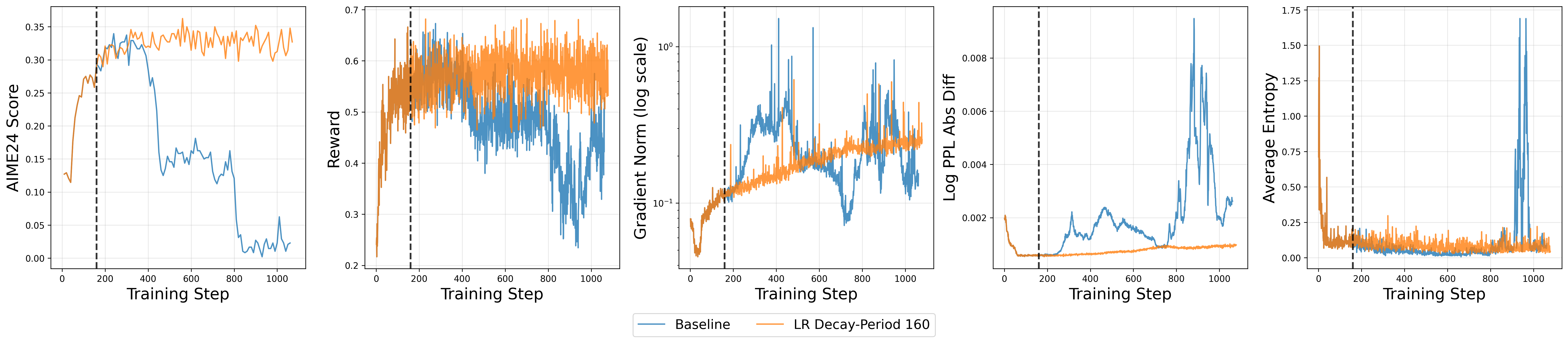}
    \caption{Comparison of applying length-decay LR Scheduler or not on Qwen3-8B-Base. The vertical dotted line is at 160 step, indicating the starting point of decaying LR.}
    \label{fig:8B-lrdecay-single}
\end{figure}

As Figure \ref{fig:4B-lrdecay-single} and \ref{fig:8B-lrdecay-single} shows, applying our customized length-decay LR scheduler clearly improves peak validation performance and stabilizes training. Our method maintains a stable gradient norm and constrains the mismatch indicator \(\log\text{ppl\_abs\_diff}\) within a safe numerical regime, effectively neutralizing engine-level discrepancies. Furthermore, the reduced learning rate helps preserve higher policy entropy, preventing premature mode-collapse and maintaining the model's diversity and adaptability for subsequent tasks. \citep{cui2025entropy,yue2025does}

\subsection{Criteria of Choosing Decay Period}

As discussed in Section \ref{Sec:Method}, the average response length exhibits a "surge" pattern, tripling within a remarkably narrow window. Since the error of the policy gradient estimate scales with sequence length, this surge injects significant noise into the optimization dynamics, necessitating a lower learning rate to maintain stability. Hence, we investigate how we should choose the key parameter \texttt{decay\_period} in our scheduler based on the time point of length surge.

For Qwen3-4B-Base model, the response length surge occurs near \textbf{110} step. We evaluated 3 different decay\_period : 125 (immediately following the surge), 250 (double the initial surge interval), and 204 (coinciding with the first epoch). As shown in Figure \ref{fig:4B-lrdecay-multi}, a \texttt{decay\_period} of 125 and 204 both successfully stabilize training, with 204 performing slightly better. In contrast, a period of 250 merely extends the stable window but fails to prevent eventual collapse.

\begin{figure}[H]
    \centering
    \includegraphics[width=1\linewidth]{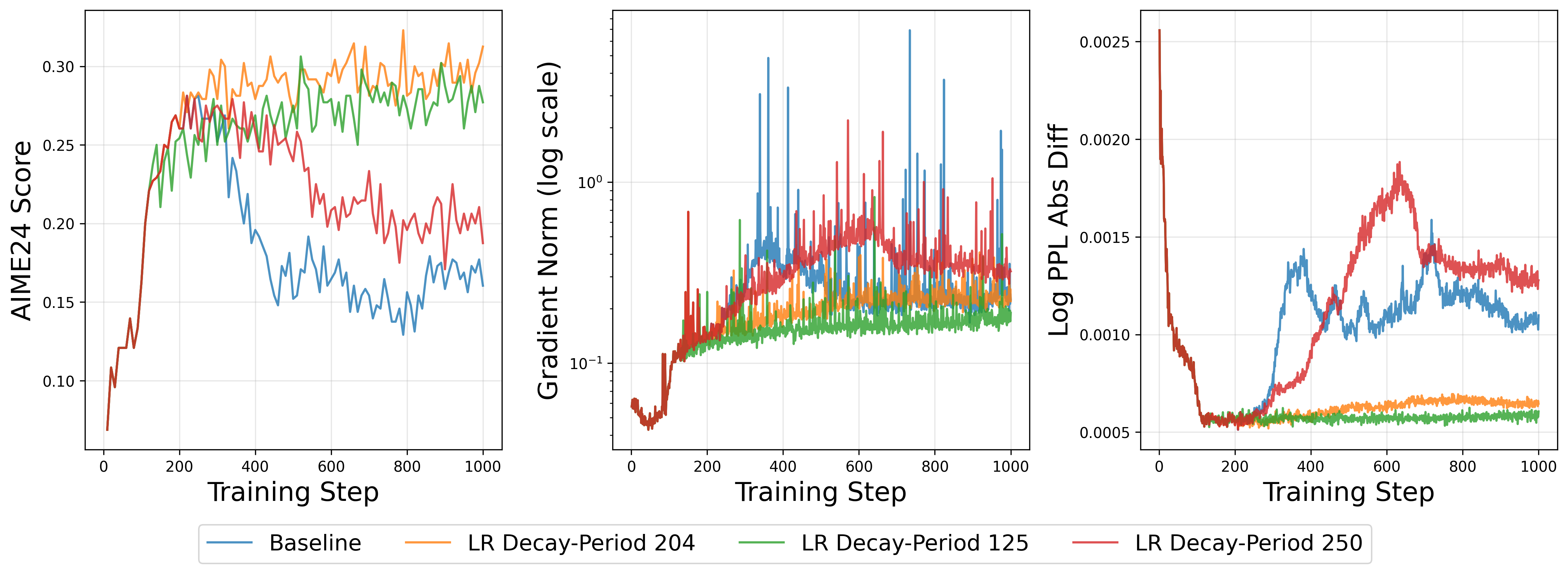}
    \caption{Comparison of using different \texttt{decay\_period} on Qwen3-4B-Base.}
    \label{fig:4B-lrdecay-multi}
\end{figure}

To decouple the optimal \texttt{decay\_period} from the epoch count, we replicated the experiment on Qwen3-8B-Base, where the surge occurs earlier, at step \textbf{90}. We tested periods of 100, 160, and 204. Notably, the epoch-aligned period (204) led to a decline in AIME24 accuracy (Figure \ref{fig:8B-lrdecay-multi}), whereas the optimal performance was achieved at 160. This value scales linearly with the surge onset by a factor of approximately $1.8$. Across both models, early decay did not significantly degrade performance, demonstrating that our length-surge-based heuristic is a more robust anchoring metric than traditional epoch counts.
\begin{figure}[H]
    \centering
    \includegraphics[width=1\linewidth]{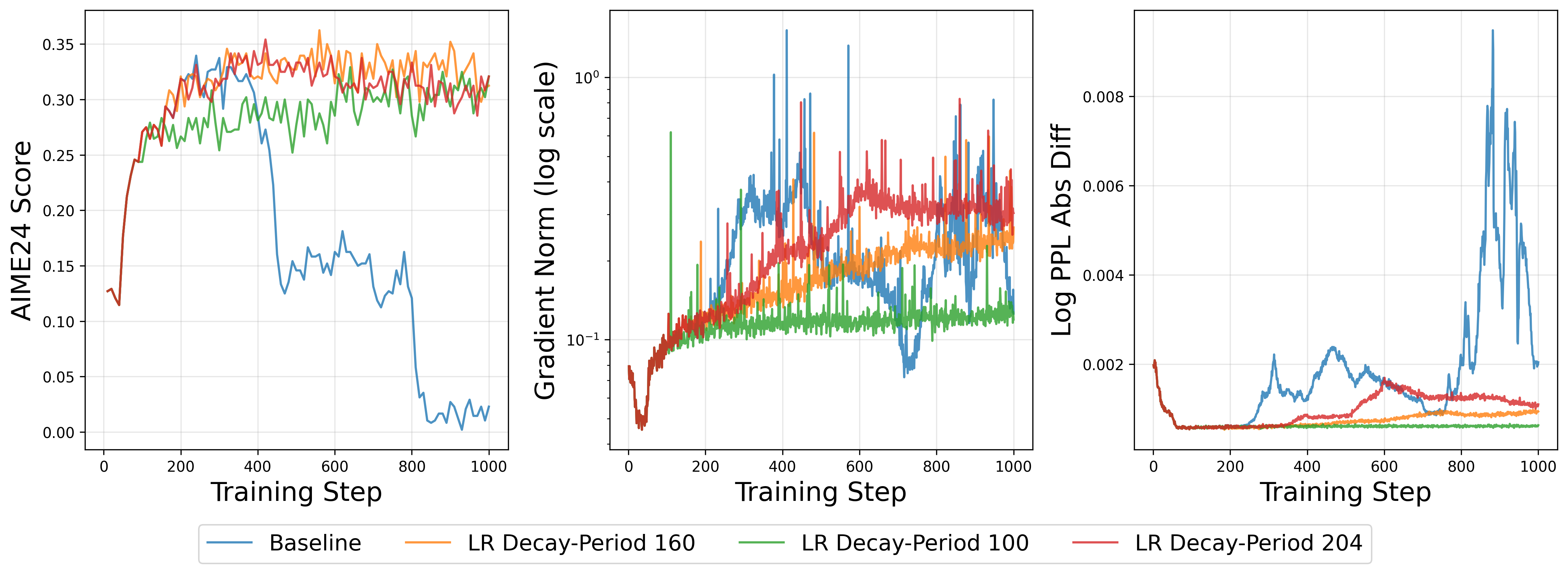}
    \caption{Comparison of using different \texttt{decay\_period} on Qwen3-8B-Base.}
    \label{fig:8B-lrdecay-multi}
\end{figure}

\subsection{Ablation Experiments on Importance Sampling}

Is a LR scheduler necessary if we already use techniques like Importance Sampling(IS)? We have partially answered this question in Section \ref{Sec:Background}: IS doesn't always work. We answer another part of this question here: LR scheduler is able to fix it. We pick token-level MIS and token-level TIS to investigate the influence of LR scheduler when IS patch fails or manages to prevent collapse.
\begin{figure}[H]
    \centering
    \includegraphics[width=1\linewidth]{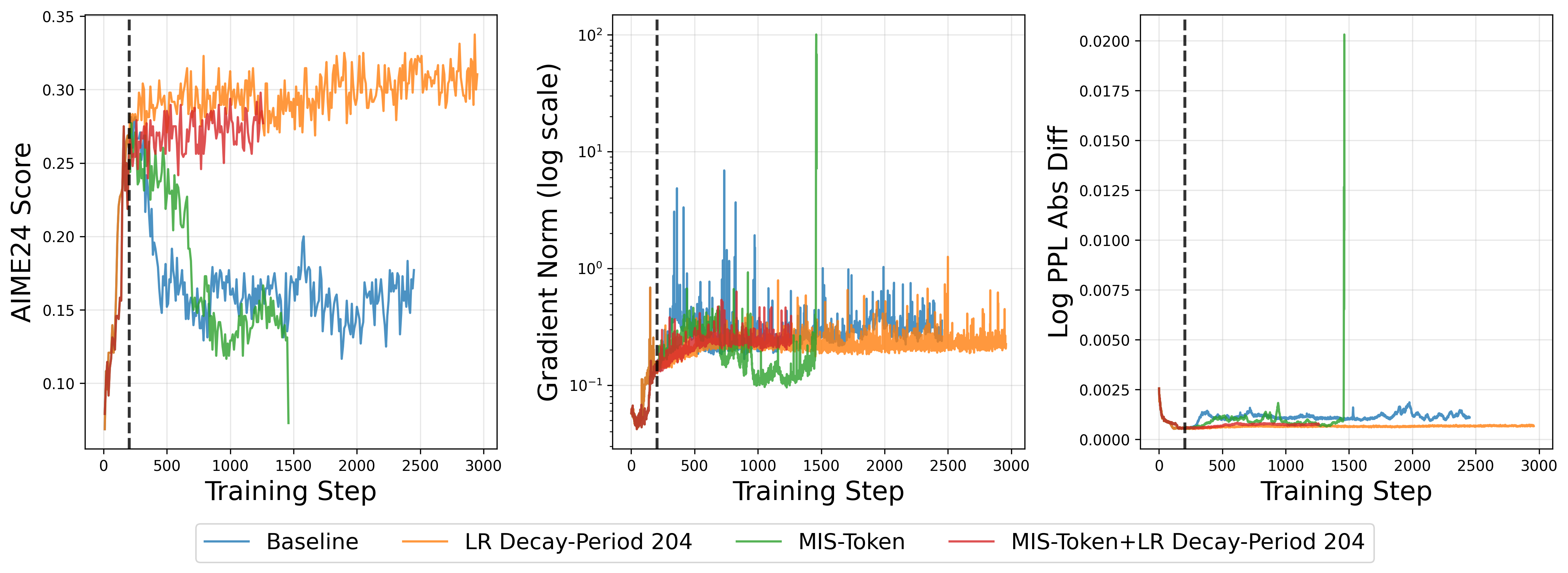}
    \caption{Ablation experiment on Token-level MIS on Qwen3-4B-Base.}
    \label{fig:Ablation:MIS-Token}
\end{figure}

\begin{figure}[H]
    \centering
    \includegraphics[width=1\linewidth]{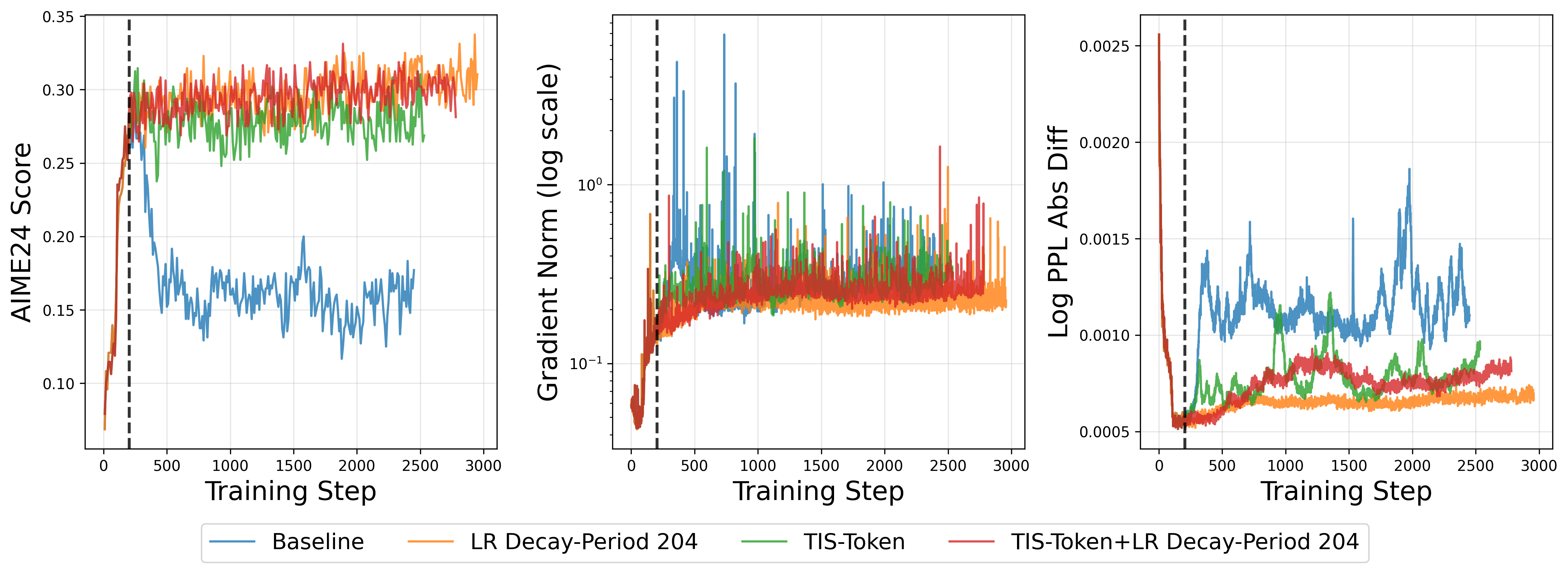}
    \caption{Ablation experiment on Token-level TIS on Qwen3-4B-Base.}
    \label{fig:Ablation:TIS-Token}
\end{figure}
As Figure \ref{fig:Ablation:MIS-Token} shows, MIS doesn't solve collapse problem, while our length-decay LR scheduler effectively stabilizes training regardless of whether MIS was used. In TIS experiment(Figure \ref{fig:Ablation:TIS-Token}), while TIS patch manages to prevent collapse, the addition of the LR scheduler further elevates validation accuracy and significantly dampens numerical spikes in the mismatch indicator.

These findings suggest that importance sampling and learning rate scheduling operate at different levels of the optimization process. While IS attempts to correct per-sample gradient bias, our scheduler addresses the fundamental dynamic instability by modulating the update magnitude. Thus, the adaptive LR scheduler serves as a robust, complementary solution that enhances stability even in the presence of existing algorithmic patches.

\subsection{Ablations on FP16 Precision}

A recent study \citep{qi2025defeating} suggests that substituting the default BF16 data type with FP16 may enhance the stability of RL training by providing finer numerical resolution in certain ranges. To evaluate this claim, we investigated whether switching to FP16 inherently resolves the training-inference mismatch and assessed the efficacy of our scheduler under this alternative precision setting.

\begin{figure}[ht]
    \centering
    \includegraphics[width=1\linewidth]{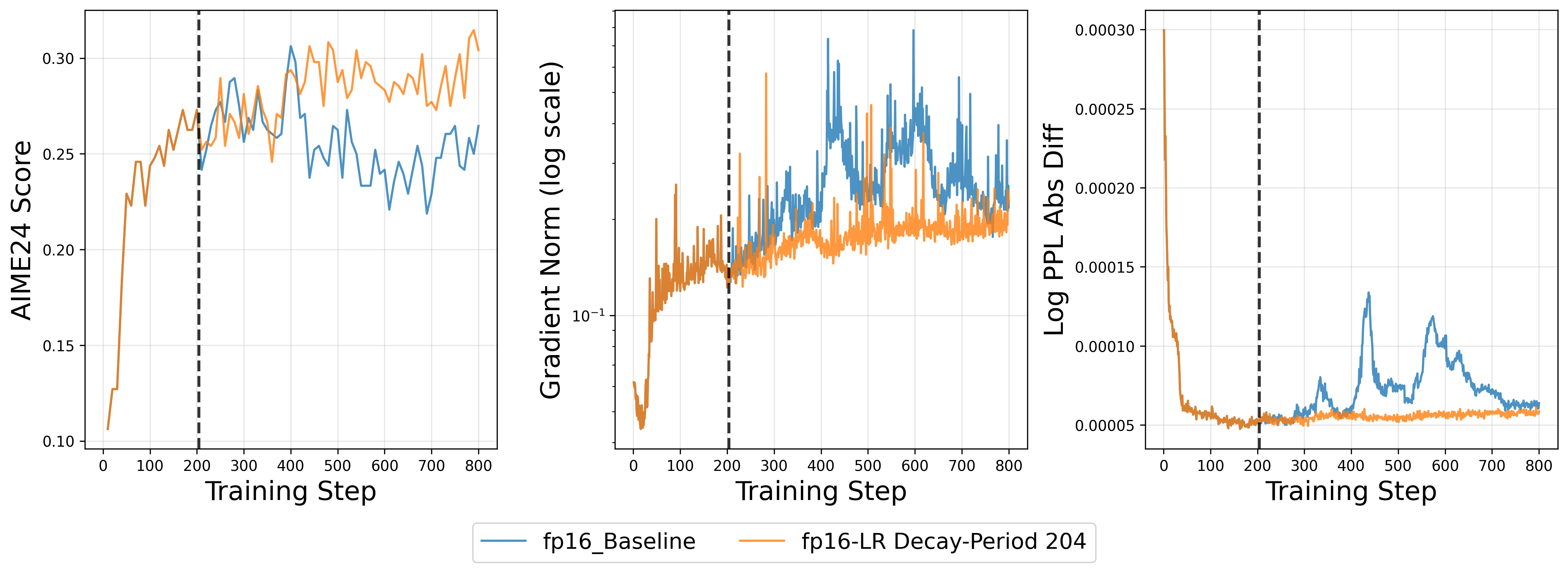}
    \caption{Experiment on fp16.}
    \label{fig:fp16}
\end{figure}

As Figure \ref{fig:fp16} shows, fp16 exhibits similar training collapse as bf16 in previous experiment. Notably, our length-triggered LR scheduler successfully stabilizes the training process using the same \texttt{decay\_period} as in previous experiments. These results reinforce our central argument: the training-inference mismatch and the resulting instability are not merely precision problems. Instead, they represent a dynamic optimization challenge that transcends low-level hardware formats, necessitating an optimization-level intervention. 

\section{Conclusion}
\label{Sec:conclusion}

In this work, we investigated the persistent instability of Reinforcement Learning for LLMs, with a specific focus on the phenomenon of training-inference mismatch. Our findings challenge the conventional characterization of this mismatch as a static numerical artifact, revealing it instead as a dynamic optimization failure deeply coupled with gradient noise and the model's evolution along its trajectory. We introduced a reactive LR scheduler triggered by the surge in average response length—a metric we identify as a robust early-warning signal for impending instability.

Empirical evaluations across model scales demonstrate that our method consistently stabilizes the training process, effectively suppresses engine-level distribution mismatch, and achieves superior peak performance compared to standard importance sampling and constant learning rate strategies. We believe these insights provide a crucial step toward understanding the mechanical origins of training-inference mismatch and offer a practical recipe for stable RL orchestration at scale.

\bibliographystyle{plainnat}
\bibliography{ref}

\newpage
\appendix
\section*{Appendix}

\section{Brief discussion on LR decay and gradient noise} \label{Appendix:Discussion}
Assuming that the objective function $\mathcal{J}(\theta)$ is $L$-smooth, we have the following quadratic lower bound:
\begin{equation}
\mathcal{J}(\theta_{k+1}) \geq \mathcal{J}(\theta_k) + \langle \nabla \mathcal{J}(\theta_k), \theta_{k+1} - \theta_k \rangle - \frac{L}{2} \|\theta_{k+1} - \theta_k\|^2
\end{equation}

Consider the stochastic gradient update rule with step size $\eta$:
\begin{equation}
\theta_{k+1} = \theta_k + \eta \hat{g}_k
\end{equation}
where $\hat{g}_k$ is a biased estimate of the true gradient $\nabla \mathcal{J}(\theta_k)$, decomposed as:
\begin{equation}
\hat{g}_k = \nabla \mathcal{J} + \mathbf{Bias} + \xi, \quad \mathbb{E}[\xi] = 0, \quad \mathbb{E}[\|\xi\|^2] = \mathbf{Var}
\end{equation}

Substituting the update rule into the smoothness inequality:
\begin{equation}
\mathcal{J}(\theta_{k+1}) - \mathcal{J}(\theta_k) \geq \eta \langle \nabla \mathcal{J}, \hat{g}_k \rangle - \frac{L\eta^2}{2} \|\hat{g}_k\|^2
\end{equation}

Taking the expectation $\mathbb{E}[\cdot]$ conditioned on $\theta_k$:
\begin{itemize}
    \item \textbf{Linear Term:}
    \begin{equation*}
    \mathbb{E}[\eta \langle \nabla \mathcal{J}, \nabla \mathcal{J} + \mathbf{Bias} + \xi \rangle] = \eta \|\nabla \mathcal{J}\|^2 + \eta \langle \nabla \mathcal{J}, \mathbf{Bias} \rangle
    \end{equation*}
    \item \textbf{Quadratic Term:}
    Using $\mathbb{E}[\|\mathbf{X}\|^2] = \|\mathbb{E}[\mathbf{X}]\|^2 + \text{Var}(\mathbf{X})$:
    \begin{equation*}
    \mathbb{E}[\|\hat{g}_k\|^2] = \|\nabla \mathcal{J} + \mathbf{Bias}\|^2 + \mathbf{Var} = \|\nabla \mathcal{J}\|^2 + 2\langle \nabla \mathcal{J}, \mathbf{Bias} \rangle + \|\mathbf{Bias}\|^2 + \mathbf{Var}
    \end{equation*}
\end{itemize}

Combining these to get the main inequality:
\begin{equation} \label{eq:noisy gradient}
\mathbb{E}[\mathcal{J}(\theta_{k+1})] - \mathcal{J}(\theta_k) \geq 
\underbrace{\eta \left( 1 - \frac{L\eta}{2} \right) \|\nabla \mathcal{J}\|^2}_{\text{Term A: True Progress}} + 
\underbrace{\eta(1 - L\eta) \langle \nabla \mathcal{J}, \mathbf{Bias} \rangle}_{\text{Term B: Bias Effect}} - 
\underbrace{\frac{L\eta^2}{2} \left[ \mathbf{Var} + \|\mathbf{Bias}\|^2 \right]}_{\text{Term C: Noise Penalty}}
\end{equation}

The right side of inequality comprises of three terms. For small $\eta$, term A and term B scales roughly linearly with $\eta$ (since $\eta \gg \eta^2$ when $\eta$ is small). Hence, reducing \( \eta \) indeed slows down your true progress and dampens bias effect(which is unknown to be positive or negative); however, it reduces term C quadratically, faster than it slows down the progress, which means decaying learning rate would be helpful in resisting influence brought by gradient noise(including bias and variance).

\section{Proof of Theorem 1} \label{Appendix:proof}

\begin{theorem} \label{thm:horizon}
    Let \(\nabla_{\theta} \mathcal{J}(\theta)\) and \(\nabla_{\theta} \mathcal{J}_{actual}(\theta)\) be the ideal gradient for backpropagation and the real calculated gradient defined in \ref{eq:ideal graident} and \ref{eq:biased gradient}, respectively. Then, let reward \(R(\tau)\) only takes value in \(0\) and \(1\), under assumptions \ref{ass: Local Numerical Drift} and \ref{ass: Bounded Score Function}, the following inequality holds:

\begin{equation*}
    \left \| \nabla_{\theta} \mathcal{J}_{actual}(\theta)-\nabla_{\theta} \mathcal{J}(\theta) \right \|_2 \leq C \cdot T^2 
\end{equation*}
    Where \(C=2B \Delta_{max}\).
\end{theorem}

\begin{assumption}[Local Numerical Drift] \label{ass: Local Numerical Drift}
   We assume the numerical discrepancy of \textbf{single-token} between engines results in a bounded Total Variation (TV) distance between their local distributions: 
   \begin{align*}
       \sup_s D_{TV}(\textcolor{blue}{\pi}_\theta(\cdot|s) \| \textcolor{red}{\mu}_\theta(\cdot|s)) \le \Delta_{max}
   \end{align*}
\end{assumption}

\begin{assumption}[Bounded Score Function] \label{ass: Bounded Score Function}
The norm of the (token-level) score function is bounded:
    \begin{align*}
        \sup_{a,s}\|\nabla_\theta \log \textcolor{blue}{\pi}_\theta(a|s)\|_2 \le B
    \end{align*}
\end{assumption}

\noindent To prove the theorem, we first introduce two lemmas:

\begin{lemma} [Conversion to token-level objective]\label{lemma: Conversion to token-level objective}

For any policy \(\textcolor{blue}{\pi}\) and \(\textcolor{red}{\mu}\), we have the following equality:
\begin{equation}
    \mathbb{E}_{x \sim \mathcal{D}, y \sim \textcolor{red}{\mu}(\cdot|x)} \left[ R(x, y) \nabla_\theta \log \textcolor{blue}{\pi}_\theta(y|x) \right] = \mathbb{E}_{s \sim d_{\textcolor{red}{\mu}}} \mathbb{E}_{a \sim \textcolor{red}{\mu}(\cdot|s)} \left[ R(s, a) \cdot \nabla_\theta \log \textcolor{blue}{\pi}_\theta(a|s) \right] 
\end{equation}

Where \(x\) is prompt, \(y\) is (sequence-level) response, and \(d_\pi(s) := \mathbb{E}_{x \sim \mathcal{D}, y' \sim \pi(\cdot|x)} \left[ \sum_{t'=0}^{|y'|-1} \mathbb{I} \{ (x, y'_{<t'}) = s \} \right] = P(x) \cdot \prod_{k=0}^{t-1} \pi(y_k | x, y_{<k})\) is state occupancy measure.
\begin{proof}
    \begin{align*}
    & \mathbb{E}_{x \sim \mathcal{D}, y \sim \textcolor{red}{\mu}(\cdot|x)} \left[ R(x, y) \nabla_\theta \log \textcolor{blue}{\pi}_\theta(y|x) \right] \\
    &= \sum_{t=0}^{\infty} \mathbb{E}_{s_t, a_t \sim \textcolor{red}{\mu}} \left[ R(s_t, a_t) \cdot \nabla_\theta \log \textcolor{blue}{\pi}_\theta(a_t|s_t) \right] \quad (\text{for empty terms beyond sequence length \(T\), take 0}) \\
    &= \sum_{t=0}^{\infty} \sum_{s \in \mathcal{S}} P(s_t = s | \textcolor{red}{\mu}) \sum_{a \in \mathcal{A}} \textcolor{red}{\mu}(a|s) \cdot \left[ R(s, a) \cdot \nabla_\theta \log \textcolor{blue}{\pi}_\theta(a|s) \right] \\
    &= \sum_{s \in \mathcal{S}} \left( \sum_{t=0}^{\infty} P(s_t = s | \textcolor{red}{\mu}) \right) \left( \sum_{a \in \mathcal{A}} \textcolor{red}{\mu}(a|s) \cdot R(s, a) \cdot \nabla_\theta \log \textcolor{blue}{\pi}_\theta(a|s) \right) \\
    &= \sum_{s \in \mathcal{S}} d_\mu(s) \cdot \mathbb{E}_{a \sim \mu(\cdot|s)} \left[ R(s, a) \cdot \nabla_\theta \log \textcolor{blue}{\pi}_\theta(a|s) \right] \quad (\text{comes from the definition of state occupancy  measure}) \\
    &= \mathbb{E}_{s \sim d_{\textcolor{red}{\mu}}} \mathbb{E}_{a \sim \textcolor{red}{\mu}(\cdot|s)} \left[ R(s, a) \cdot \nabla_\theta \log \textcolor{blue}{\pi}_\theta(a|s) \right] 
    \end{align*}
\end{proof}

\end{lemma}

This lemma is used to connect sequence-level gradient with our assumption on the upper bound of token-level mismatch, and the next one is used to bound the token-level discrepancy growth in auto-regressive generation:

\begin{lemma} [State Occupancy Drift]\label{lemma: State Occupancy Drift}
    Let $d_t^{\textcolor{blue}{\pi}}$ and $d_t^{\textcolor{red}{\mu}}$ be the state occupancy distributions at time $t$ for policies $\textcolor{blue}{\pi}$ and $\textcolor{red}{\mu}$ respectively. The $L_1$ distance between these distributions accumulates linearly over time:
\begin{equation}
    \delta_t = \|d_{\textcolor{blue}{\pi},t} - d_{\textcolor{red}{\mu},t}\|_1 \le  2t \cdot \Delta_{max}
\end{equation}

\begin{proof}

Using the state transition $P(s'|s,a)$:
\begin{align*}
\delta_t &= \sum_{s'} \left| \sum_{s,a} P(s'|s,a) \textcolor{blue}{\pi}(a|s) d_{\textcolor{blue}{\pi},t-1}(s) - \sum_{s,a} P(s'|s,a) \textcolor{red}{\mu}(a|s) d_{\textcolor{red}{\mu},t-1}(s) \right| \\
&\leq \sum_{s'} \left| \sum_{s,a} P(s'|s,a) d_{\textcolor{blue}{\pi},t-1}(s) (\textcolor{blue}{\pi}(a|s) - \textcolor{red}{\mu}(a|s)) \right| + \sum_{s'} \left| \sum_{s,a} P(s'|s,a) \textcolor{red}{\mu}(a|s) (d_{\textcolor{blue}{\pi},t-1}(s) - d_{\textcolor{red}{\mu},t-1}(s)) \right| \\
&\leq \sum_{s,a} \left( \sum_{s'} P(s'|s,a) \right) d_{\textcolor{blue}{\pi},t-1}(s) |\textcolor{blue}{\pi}(a|s) - \textcolor{red}{\mu}(a|s)| + \sum_{s} \left( \sum_a \textcolor{red}{\mu}(a|s) \sum_{s'} P(s'|s,a) \right) |d_{\textcolor{blue}{\pi},t-1}(s) - d_{\textcolor{red}{\mu},t-1}(s)| \\
&\leq \delta_{t-1} + 2\Delta_{max}
\end{align*}

Thus, $\delta_t = \|d_{\textcolor{blue}{\pi},t} - d_{\textcolor{red}{\mu},t}\|_1 \le 2t \cdot \Delta_{max}$.
\end{proof}

\end{lemma}

\noindent Now we turn to the proof of Theorem \ref{thm:horizon}:
\begin{proof}
    \begin{align*}
         &\left \| \nabla_{\theta} \mathcal{J}_{actual}(\theta)-\nabla_{\theta} \mathcal{J}(\theta) \right \| \\ =&\left \|  \mathbb{E}_{x \sim p_x}  \mathbb{E}_{y \sim \textcolor{blue}{\pi}(\cdot|x,\theta)} \left[ \nabla_{\theta} \log \textcolor{blue}{\pi}(y|x,\theta) R(x,y)\right]-  \mathbb{E}_{x \sim p_x}  \mathbb{E}_{y \sim \textcolor{red}{\mu}(\cdot|x,\theta)} \left[ \nabla_{\theta} \log \textcolor{blue}{\pi}(y|x,\theta) R(x,y)\right]  \right \|_2 \\
         =& \left \| \sum_{t=0}^{T-1} \sum_s (d_{\textcolor{blue}{\pi}, t}(s) \sum_{a}R(s, a)\textcolor{blue}{\pi}_\theta(a|s) \nabla_\theta \log \textcolor{blue}{\pi}_\theta(a|s) - d_{\textcolor{red}{\mu}, t}(s)\sum_{a}R(s, a)\textcolor{red}{\mu}_\theta(a|s) \nabla_\theta \log \textcolor{blue}{\pi}_\theta(a|s))  \right \|_2 \\
         \leq&  \sum_{t=0}^{T-1} \sum_s  \sum_{a} \left \|d_{\textcolor{blue}{\pi}, t}(s) R(s, a)\textcolor{blue}{\pi}_\theta(a|s) \nabla_\theta \log \textcolor{blue}{\pi}_\theta(a|s) - d_{\textcolor{red}{\mu}, t}(s)R(s, a)\textcolor{red}{\mu}_\theta(a|s) \nabla_\theta \log \textcolor{blue}{\pi}_\theta(a|s)  \right \|_2 \\
         \leq& B\sum_{t=0}^{T-1} \sum_s  \sum_{a} \left |d_{\textcolor{blue}{\pi}, t}(s)\textcolor{blue}{\pi}_\theta(a|s)  - d_{\textcolor{red}{\mu}, t}(s)\textcolor{red}{\mu}_\theta(a|s)  \right | \\
         \leq& B(\sum_{t=0}^{T-1} \sum_s  \sum_{a} \left |d_{\textcolor{blue}{\pi}, t}(s)\textcolor{blue}{\pi}_\theta(a|s)  - d_{\textcolor{blue}{\pi}, t}(s)\textcolor{red}{\mu}_\theta(a|s)  \right | + \sum_{t=0}^{T-1} \sum_s  \sum_{a} \left |d_{\textcolor{blue}{\pi}, t}(s)\textcolor{red}{\mu}_\theta(a|s)  - d_{\textcolor{red}{\mu}, t}(s)\textcolor{red}{\mu}_\theta(a|s)  \right |) \\
         \leq& B(\sum_{t=0}^{T-1} 2\Delta_{max}+ \sum_{t=0}^{T-1} \delta_{t}) \\
         \leq& 2B\Delta_{max} \cdot T^2
    \end{align*}
\end{proof}

\section{Hyperparameters for reproducibility} \label{Appendix:hyperparameter}

Here we provide some key configurations in our main experiment for reproducibility purpose.
\begin{table}[htbp]
\centering
\caption{Key configurations in experiment}
\label{tab:ppo_hyperparams}
\begin{tabularx}{\textwidth}{l p{4.5cm} X}
\toprule
\textbf{Category} & \textbf{Hyperparameter} & \textbf{Description \& Value} \\
\midrule
\textbf{Optimization} 
                    & \texttt{Initial (Actor) LR} & 1e-6 \\
                    & \texttt{weight\_decay} & 0.01 \\
                    & \texttt{grad\_clip\_threshold} & 1 \\
                    & \texttt{(train)batch\_size} & 64 \\
                    & \texttt{ppo\_mini\_batch\_size} & 64, i.e fully on-policy \\
                    & \texttt{ppo\_epoch} & 1 \\
                    & \texttt{Min\_LR\_Ratio} & 0.1 \\
\midrule
\textbf{Algorithm}  
                    & \texttt{Advantage estimator type} & rloo \\
                    & \texttt{clip\_range} ($\epsilon$) & 0.28(high)/0.2(low) \\
                    & \texttt{Discount factor} ($\gamma$) & 1 \\
                    & \texttt{KL coefficient} ($\beta$) & 0 \\
                    & \texttt{TIS threshold} ($C$) & 2 \\
                    & \texttt{MIS threshold} ($C$) & 2 \\
                    & \texttt{Rejection sampling} & True(Ignore the samples with all 0/1 rewards) \\

\midrule
\textbf{Resource}   
                    & \texttt{Rollouts per prompt($n$)} & 16 \\
                    & \texttt{mixed\_precision\_training} & True \\
                    & \texttt{max\_response\_length} & 8192 \\
                    & \texttt{n\_gpus\_per\_node} & 8 \\
                    & \texttt{nnode} & 2 \\
                    & \texttt{GPU type} & H100 \\
\midrule
\textbf{Others}
                    & \texttt{Training engine} & fsdp \\
                    & \texttt{Inference engine} & vllm \\
                    & \texttt{Rollout Temperature} & 1 \\
\bottomrule
\end{tabularx}
\end{table}
\end{document}